\documentclass[12pt]{article}
\usepackage[latin1]{inputenc}
\usepackage{cmap}
\usepackage{lmodern}

\usepackage{amssymb, amsmath, amsthm}
\usepackage[a4paper,top=25mm,bottom=25mm,left=25mm,right=25mm]{geometry}
\usepackage{etex}

\usepackage{authblk} 
\usepackage{pifont}
\usepackage{graphicx}
\usepackage[usenames,dvipsnames,svgnames,table]{xcolor}
\usepackage[figuresright]{rotating}
\usepackage{xtab} 
\usepackage{longtable} 
\usepackage{multirow}
\usepackage{footnote}
\usepackage[stable]{footmisc}
\usepackage{chngpage} 
\usepackage{pdflscape} 

\usepackage{pgfplots}
\pgfplotsset{compat=1.14}
\usepackage{setspace}

\makesavenoteenv{tabular}
\usepackage{tabularx}
\usepackage{booktabs}
\usepackage{threeparttable}
\usepackage[referable]{threeparttablex} 
\newcolumntype{R}{>{\raggedleft\arraybackslash}X}
\newcolumntype{L}{>{\raggedright\arraybackslash}X}
\newcolumntype{C}{>{\centering\arraybackslash}X}
\newcolumntype{A}{>{\columncolor{gray!25}}C}
\newcolumntype{a}{>{\columncolor{gray!25}}c}

\usepackage{dcolumn} 
\newcolumntype{.}{D{.}{.}{-1}}

\usepackage{tikz}
\usetikzlibrary{arrows}
\usepackage[semicolon]{natbib}
\usepackage[hyphens]{url}
\usepackage{hyperref} 
\hypersetup{
  colorlinks   = true,    
  urlcolor     = blue,    
  linkcolor    = blue,    
  citecolor    = red      
}
\usepackage{microtype}
\usepackage[justification=centerfirst]{caption}

\usepackage[labelformat=simple]{subcaption}

\DeclareCaptionLabelFormat{parenthesis}{(#2)}
\makeatletter
\renewcommand\p@subfigure{\arabic{figure}.}
\makeatother

\DeclareCaptionLabelFormat{parenthesis}{(#2)}
\makeatletter
\renewcommand\p@subtable{A.\arabic{table}.}
\makeatother

\usepackage{enumitem}

\setlist[itemize]{leftmargin=3\parindent}
\setlist[enumerate]{leftmargin=2\parindent}

\theoremstyle{plain}

\newtheorem{lemma}{Lemma}

\newtheorem{theorem}{Theorem}

\theoremstyle{definition}
\newtheorem{axiom}{Axiom}

\newtheorem{definition}{Definition}
\newtheorem{example}{Example}

\theoremstyle{remark}

\newtheorem{remark}{Remark}

\def\keywords{\vspace{.5em} 
{\noindent \textit{Keywords}:\,}}

\def\JEL{\vspace{.5em} 
{\noindent \textbf{\emph{JEL} classification number}:\,}}

\def\AMS{\vspace{.5em} 
{\noindent \textbf{\emph{MSC} classes}:\,}}

\author{L\'aszl\'o Csat\'o\thanks{~e-mail: laszlo.csato@uni-corvinus.hu} }
\affil{Institute for Computer Science and Control, Hungarian Academy of Sciences (MTA SZTAKI) \\
Laboratory on Engineering and Management Intelligence, Research Group of Operations Research and Decision Systems}
\affil{Corvinus University of Budapest (BCE) \\
Department of Operations Research and Actuarial Sciences}
\affil{Budapest, Hungary}
\title{Characterization of an inconsistency ranking for pairwise comparison matrices}
\date{\today}

\begin{document}

\maketitle

\begin{abstract}
Pairwise comparisons between alternatives are a well-known method for measuring preferences of a decision-maker. Since these often do not exhibit consistency, a number of inconsistency indices has been introduced in order to measure the deviation from this ideal case.
We axiomatically characterize the inconsistency ranking induced by the Koczkodaj inconsistency index: six independent properties are presented such that they determine a unique linear order on the set of all pairwise comparison matrices.


\JEL{C44}

\AMS{90B50, 91B08}
\end{abstract}

\section{Introduction} \label{Sec1}

Pairwise comparisons are a fundamental tool in many decision-analysis methods such as the Analytic Hierarchy Process (AHP) \citep{Saaty1980}.
However, when different entities\footnote{~Throughout paper, the term \emph{entity} is used for the 'things' that are compared. They are sometimes called alternatives, objects, etc.} are compared with regard to abstract, non-measurable criteria by fallible humans, it may happen that the set of comparisons is not consistent: for example, entity $A$ is two times better than entity $B$, entity $B$ is three times better than entity $C$, but entity $A$ is not six ($=2 \times 3$) times better than entity $C$. Inconsistency can also be an inherent feature of the data (see e.g. \citet{BozokiCsatoTemesi2016}).

Therefore, measuring the inconsistency of pairwise comparison matrices -- that is, assigning a numerical value to the question 'How much a given matrix deviates from one representing consistent preferences?' -- emerges as an important issue. According to our knowledge, the first concept of inconsistency was presented by \citet{KendallSmith1940}.
Since then, a large number of inconsistency indices has been proposed such as Saaty's eigenvalue-based index \citep{Saaty1977}, the Koczkodaj inconsistency index \citep{Koczkodaj1993, DuszakKoczkodaj1994}, the relative error \citep{Barzilai1998}, the geometric consistency index \citep{AguaronMoreno-Jimenez2003}, the inconsistency index by Pel\'aez and Lamata \citep{PelaezLamata2003}, or the harmonic consistency index \citep{SteinMizzi2007}.

Deeper understanding of inconsistency indices is not only a mathematical exercise. \citet{BrunelliCanalFedrizzi2013} and \citet{Brunelli2016a} show that the application of various indices may lead to different, even almost opposite conclusions, so choosing among them properly really counts.
It is not by chance that a number of articles deals with the study and comparison of different inconsistency indices (see e.g. \citet{BozokiRapcsak2008} and \citet{BrunelliCanalFedrizzi2013}). Recently, an axiomatic approach has been followed: the introduction and justification of reasonable properties may help to narrow the general definition of inconsistency index as well as highlight those problematic indices that do not satisfy certain requirements \citep{BrunelliFedrizzi2011, Brunelli2016a, Brunelli2017, BrunelliFedrizzi2015, CavalloDApuzzo2012, KoczkodajSzwarc2014, KoczkodajSzybowski2015}.

Nonetheless, while logical consistency (Is the set of properties devoid of logical contradiction? Does there exist any inconsistency index that satisfies all axioms?) and independence (Is there any redundant axiom, implied by some of the others?) are usually discussed by the authors \citep{Brunelli2017, BrunelliFedrizzi2015}, we do not know any result on the \emph{axiomatic characterization} of inconsistency indices.

Characterization means proposing a set of properties such that there exists a \emph{unique} measure satisfying all requirements.
It is an extensively used approach in social choice, illustrated by the huge literature, for instance, on the axiomatization of the Shapley value in game theory (see e.g. \citet{Shapley1953, Dubey1975, Young1985, HartMas-Colell1989, vandenBrink2002}), or on the Hirsch index in scientometrics (see e.g. \citet{Woeginger2008, Quesada2010, Miroiu2013, BouyssouMarchand2014}).
It is not unknown in the case of methods used for deriving weights from pairwise comparison matrices \citep{Fichtner1984, Fichtner1986, Barzilai1997}, too.

The central contribution of this study is a characterization of the so-called Koczkodaj inconsistency ranking, induced by the Koczkodaj inconsistency index. The first axiomatization on this field suggests that inconsistency ranking may be a more natural concept than inconsistency index, since any monotonic transformation of the latter results in the same ranking. In other words, it makes no sense to differentiate between these variants of a given inconsistency index (see also \citet{Brunelli2016b}).

It is worth to note that our approach is somewhat different from the previous ones in this topic, which aimed at finding a set of suitable properties to be satisfied by any reasonable inconsistency measure.
Providing an axiomatic characterization implies neither we accept all properties involved as wholly justified and unquestionable nor we reject the axioms proposed by others.
Therefore, the current paper does not argue against the application of other inconsistency indices. The sole implication of our result is that if one agrees with these -- rather theoretic -- axioms, then the Koczkodaj inconsistency ranking remains the only choice.

Hence we follow characterizations on other fields by assigning modest importance for the motivation behind the properties used and for the issue of inconsistency thresholds.
Nonetheless, the Koczkodaj index seems to be a reasonable inconsistency index as it meets all reasonable axioms suggested by \citet{BrunelliFedrizzi2015, Brunelli2017, KoczkodajSzybowski2015}. In our view, it mitigates the problem caused by the possible hard verification of our properties in practice.
Note also that only one index has been associated to a general level of acceptance. However, if a threshold of acceptable inconsistency is given for the Koczkodaj index, then one can find the minimal number of matrix elements to be modified in order to get an acceptable matrix \citep{BozokiFulopPoesz2015}.

The paper is organized as follows. Section~\ref{Sec2} presents the main notions and notations. Section~\ref{Sec3} introduces the axioms used in the characterization of the Koczkodaj inconsistency ranking offered by Section~\ref{Sec4}. Independence of the axioms is shown in Section~\ref{Sec5}. Finally, Section~\ref{Sec6} discusses our results and highlights some directions of future research.


\section{Preliminaries} \label{Sec2}

Let $\mathbb{R}^{n \times n}_+$ denote the set of positive matrices (with all elements greater than zero) of size $n \times n$.

\begin{definition}
\emph{Pairwise comparison matrix}:
Matrix $\mathbf{A} = \left[ a_{ij} \right] \in \mathbb{R}^{n \times n}_+$ is a \emph{pairwise comparison matrix} if $a_{ji} = 1/a_{ij}$ for all $1 \leq i,j \leq n$.
\end{definition}

In a pairwise comparison matrix, $a_{ij}$ is an assessment of the relative importance of the $i$th entity with respect to the $j$th one.

\begin{definition}
\emph{Pairwise comparison submatrix}: 
Let $2 \leq m < n$. Matrix $\mathbf{B} = \left[ b_{ij} \right] \in \mathbb{R}^{m \times m}_+$ is a \emph{pairwise comparison submatrix} of pairwise comparison matrix $\mathbf{A} = \left[ a_{ij} \right] \in \mathbb{R}^{n \times n}_+$ if there exists an injection $\sigma: \{ 1;\,2;\dots;\,m \} \to \{ 1;\,2;\dots;\,n \}$ such that $b_{ij} = a_{\sigma(i) \sigma(j)}$ for all $1 \leq i,j \leq m$.
This relation is denoted by $\mathbf{B} \subset \mathbf{A}$.
\end{definition}

\begin{definition}
\emph{Triad}:
Pairwise comparison submatrix $\mathbf{B} \in \mathbb{R}^{m \times m}_+$ of pairwise comparison matrix $\mathbf{A} \in \mathbb{R}^{n \times n}_+$ is a \emph{triad} of matrix $\mathbf{A}$ if $m=3$.
\end{definition}

The term triad will be analogously used for a pairwise comparison matrix of size $3 \times 3$.

A triad $\mathbf{T}$ can be entirely described by the three elements in its upper triangle: $\mathbf{T} = (t_1;\,t_2;\,t_3)$ such that $t_1$ is the outcome of the comparison between the $1$st and the $2$nd, $t_2$ is the outcome of the comparison between the $1$st and the $3$rd, and $t_3$ is the outcome of the comparison between the $2$nd and the $3$rd entities, respectively.  

\begin{definition}
\emph{Consistency}:
Pairwise comparison matrix $\mathbf{A}  = \left[ a_{ij} \right] \in \mathbb{R}^{n \times n}_+$ is \emph{consistent} if $a_{ik} = a_{ij} a_{jk}$ for all $1 \leq i,j,k \leq n$.
\end{definition}

When the consistency condition does not hold, the pairwise comparison matrix is said to be \emph{inconsistent}.

\begin{remark}
Triad $\mathbf{T} = (t_1;\,t_2;\,t_3)$ is consistent if and only if $t_2 = t_1 t_3$.
\end{remark}

Note that a pairwise comparison matrix is consistent if and only if all of its triads are consistent.

The set of all pairwise comparison matrices is denoted by $\mathcal{A}$. Inconsistency is measured by associating a value for each pairwise comparison matrix.

\begin{definition}
\emph{Inconsistency index}:
Function $I: \mathcal{A} \to \mathbb{R}$ is an \emph{inconsistency index}.
\end{definition}

The paper discusses a specific inconsistency indicator, which is the following.

\begin{definition} \label{Def6}
\emph{Koczkodaj inconsistency index} \citep{Koczkodaj1993, DuszakKoczkodaj1994}:
Consider a pairwise comparison matrix $\mathbf{A} = \left[ a_{ij} \right] \in \mathbb{R}^{n \times n}_+$. Its inconsistency according to the index $I^K$ is
\begin{equation} \label{eq1}
I^K(\mathbf{A}) = \max_{i<j<k} \left( \min \left\{ \left| 1 - \frac{a_{ik}}{a_{ij} a_{jk}} \right|; \, \left| 1 -  \frac{a_{ij} a_{jk}}{a_{ik}} \right| \right\} \right)
\end{equation}
\end{definition}

Note that the Koczkodaj inconsistency index is based on the maximally inconsistent triad such that its inconsistency is measured by its deviation from a consistent triad.

\emph{Inconsistency ranking} $\succeq$ is a linear order (complete, antisymmetric and transitive binary relation) on the set of all pairwise comparison matrices $\mathcal{A}$ with respect to inconsistency.
$\mathbf{A} \succeq \mathbf{B}$ is interpreted such that pairwise comparison matrix $\mathbf{A}$ is at most as inconsistent as matrix $\mathbf{B}$.
The relations $\sim$ and $\succ$ are derived from $\succeq$ in the usual way.

Any inconsistency index induces an inconsistency ranking: $\mathbf{A} \succeq \mathbf{B}$ if and only if matrix $\mathbf{A}$ has a not worse (typically, not larger) value of inconsistency than matrix $\mathbf{B}$ according to the given index.

\begin{definition} \label{Def7}
\emph{Koczkodaj inconsistency ranking}:
Consider pairwise comparison matrices $\mathbf{A} = \left[ a_{ij} \right] \in \mathbb{R}^{n \times n}_+$ and $\mathbf{B} = \left[ b_{ij} \right] \in \mathbb{R}^{m \times m}_+$. Then $\mathbf{A} \succeq^K \mathbf{B}$ if
\begin{equation} \label{eq2}
\max_{i<j<k} \left( \max \left\{ \frac{a_{ij} a_{jk}}{a_{ik}} ; \, \frac{a_{ik}}{a_{ij} a_{jk}} \right\} \right) \leq \max_{i<j<k} \left( \max \left\{ \frac{b_{ij} b_{jk}}{b_{ik}} ; \, \frac{b_{ik}}{b_{ij} b_{jk}} \right\} \right).
\end{equation}
\end{definition}

The Koczkodaj inconsistency ranking is the inconsistency ranking induced by the Koczkodaj inconsistency index since, for any triad $\mathbf{T} = (t_1; t_2; t_3)$, there exists a differentiable one-to-one correspondence between the Koczkodaj inconsistency index and the inconsistency index $I^T(\mathbf{T}) = \max \left\{ (t_1 t_3) / t_2 ; \, t_2 / (t_1 t_3) \right\}$ \citep[Theorem ~3.1]{BozokiRapcsak2008}.

\section{Axioms for an inconsistency ranking} \label{Sec3}

We start by giving some properties that one could expect an inconsistency ranking to satisfy. After that, they will be commented.

\begin{axiom}
\emph{Positive responsiveness} ($PR$):
Consider two triads $\mathbf{S} = (1; \,s_2; \,1)$ and $\mathbf{T} = (1; \,t_2; \,1)$ such that $s_2,t_2 \geq 1$.
Inconsistency ranking $\succeq$ satisfies $PR$ if $\mathbf{S} \succeq \mathbf{T} \iff s_2 \leq t_2$.
\end{axiom}

\begin{axiom}
\emph{Invariance under inversion of preferences} ($IIP$):
Consider a triad $\mathbf{T}$ and its transpose $\mathbf{T}^\top$.
Inconsistency ranking $\succeq$ satisfies $IIP$ if $\mathbf{T} \sim \mathbf{T}^\top$.
\end{axiom}

\begin{axiom}
\emph{Homogeneous treatment of entities} ($HTE$):
Consider the triads $\mathbf{T} = (1; \,t_2; \,t_3)$ and $\mathbf{T'} = (1; \,t_2 / t_3; \,1)$.
Inconsistency ranking $\succeq$ satisfies $HTE$ if $\mathbf{T} \sim \mathbf{T'}$.
\end{axiom}

\begin{axiom}
\emph{Scale invariance} ($SI$):
Take the triads $\mathbf{T} = (t_1; \,t_2; \,t_3)$ and $\mathbf{T'} = (kt_1; \,k^2 t_2; \,kt_3)$.
Inconsistency ranking $\succeq$ satisfies $SI$ if $\mathbf{T} \sim \mathbf{T'}$ for all $k>0$.
\end{axiom}

\begin{axiom}
\emph{Monotonicity} ($MON$):
Consider a pairwise comparison matrix $\mathbf{A}$ and its triad $\mathbf{T}$. 
Inconsistency ranking $\succeq$ satisfies $MON$ if $\mathbf{A} \preceq \mathbf{T}$.
\end{axiom}

\begin{axiom}
\emph{Reducibility} ($RED$):
Consider a pairwise comparison matrix $\mathbf{A}$.
Inconsistency ranking $\succeq$ satisfies $RED$ if $\mathbf{A}$ has a triad $\mathbf{T}$ such that $\mathbf{A} \sim \mathbf{T}$.
\end{axiom}

Positive responsiveness is based on the observation that triad $\mathbf{T}$ differs more from the basic consistent triad of $(1; \,1; \,1)$ than triad $\mathbf{S}$. It is a similar, but more relaxed requirement than monotonicity on single comparisons, proposed by \citet{BrunelliFedrizzi2015}: $PR$ contains only one implication instead of two, and the condition is demanded for triads, not for all pairwise comparison matrices.\footnote{~The term 'monotonicity' is not used in the name of this property in order to avoid confusion with $MON$.}

Invariance under inversion of preferences means that inverting all the preferences does not affect inconsistency. This axiom was introduced in \citet{Brunelli2017} for all pairwise comparison matrices.
A fundamental reason for this requirement is that $a_{ij}$ is an assessment of the relative importance of the $i$th entity with respect to the $j$th one. However, it may equivalently reflect the relative importance of the $j$th entity with respect to the $i$th one. The two values are logically reciprocal, but one gets the pairwise comparison matrix $\mathbf{A}^\top$ instead of $\mathbf{A}$ \citep{BozokiRapcsak2008}.
Therefore, triads $\mathbf{T}$ and $\mathbf{T}^\top$ should be in the same equivalence class of inconsistency.

Homogeneous treatment of entities can be justified in the following way. According to triads $\mathbf{T}$ and $\mathbf{T'}$, we have two entities (the 1st and the 2nd) equally important on their own. After they are compared with a third entity, the inconsistency of the resulting triad should not depend on the relative importance of the new entity. The violation of $HTE$ would mean that inconsistency is influenced by the absolute importance of entities, contradicting to the notion of pairwise comparisons.

Scale invariance means that the level of inconsistency is independent from the mathematical representation of the preferences. For example, take the $i$th, the $j$th and the $k$th entities such that the $i$th is 'moderately more important' than $j$th and $j$th is 'moderately more important' than the $k$th with the same intensity of the two relations. Then it makes sense to expect the inconsistency ranking to be the same if 'moderately more important' is identified by the number $2$, or $3$, or $4$, and so on, even allowing for a change in the direction of the two preferences.
Naturally, this transformation should preserve consistency, which finalizes the definition of $SI$.

Monotonicity is formulated on the idea that 'no pairwise comparison submatrix may have a worse inconsistency indicator than the given pairwise comparison matrix' by Jacek Szybowski in \citet{KoczkodajSzybowski2015}. Hence an extension of the set of entities compared cannot improve on inconsistency.
For our characterization, it is enough to expect that no triad of the original matrix is allowed to be more inconsistent.

Reducibility ensures that there exists a 'critical' triad whose inconsistency is responsible for the inconsistency of the original pairwise comparison matrix.
It will be a crucial axiom in the main result, simplifying the problem by conmtracting its size. Readers familiar with characterizations in other fields of science will not be surprised: for example, the reduced game property is widely used in axiomatizations of different game theory concepts (see e.g. \citet{DavisMaschler1965}).

Note that $MON$ and $RED$ together imply monotonicity not only for triads, but for all pairwise comparison submatrices, similarly to the definition of monotonicity by \citet{KoczkodajSzybowski2015}.

\section{Characterization of the Koczkodaj inconsistency ranking} \label{Sec4}

\begin{theorem} \label{Theo1}
The Koczkodaj inconsistency ranking is the unique inconsistency ranking satisfying positive responsiveness, invariance under inversion of preferences, homogeneous treatment of entities, scale invariance, monotonicity and reducibility.
\end{theorem}

\begin{proof}
We will first argue that the Koczkodaj inconsistency ranking satisfies all axioms.
$PR$, $HTE$ and $SI$ immediately follows from \eqref{eq2}.
$IIP$ is met due to taking the maximum of $a_{ij} a_{jk} / a_{ik}$ and its inverse in \eqref{eq2}.
The Koczkodaj inconsistency ranking considers only the maximally inconsistent triad, and the set of triads of a pairwise comparison submatrix is a subset of the set of triads of the original matrix, therefore it meets $MON$ and $RED$.

Uniqueness: it is shown that any inconsistency ranking satisfying the six axioms coincides with the Koczkodaj inconsistency ranking.

Assume that there exist pairwise comparison matrices $\mathbf{A}$ and $\mathbf{B}$ such that $\mathbf{A} \succeq \mathbf{B}$ according to an inconsistency ranking $\succeq$, which meets $PR$, $IIP$, $HTE$, $SI$, $MON$ and $RED$. 
The idea is to gradually simplify -- with the use of the axioms -- the comparison of inconsistencies of pairwise comparison matrices $\mathbf{A}$ and $\mathbf{B}$ by considering matrices with the same inconsistency level until their ranking depends on a single mathematical relation.
\begin{enumerate}
\item
Due to $RED$, there exist triads $\mathbf{A_1} = (a_1; \,a_2; \,a_3)$ and $\mathbf{B_1} = (b_1; \,b_2; \,b_3)$ such that $\mathbf{A_1} \sim \mathbf{A}$ and $\mathbf{B_1} \sim \mathbf{B}$.
\item
$MON$ provides that $\mathbf{A_1} \preceq \mathbf{A'}$ for all triads $\mathbf{A'} \subset \mathbf{A}$ and $\mathbf{B_1} \preceq \mathbf{B'}$ for all triads $\mathbf{B'} \subset \mathbf{B}$, respectively.
\item
Consider $\mathbf{A_2} = (1; \,a_2 / a_1^2; \,a_3 / a_1)$ and $\mathbf{B_2} = (1; \,b_2 / b_1^2; \,b_3 / b_1)$. $SI$ implies that $\mathbf{A_2} \sim \mathbf{A_1}$ and $\mathbf{B_2} \sim \mathbf{B_1}$.
\item
Consider $\mathbf{A_3} = \left(1; \,a_2 / (a_1 a_3); \,1 \right)$ and $\mathbf{B_3} = \left( 1; \,b_2 / (b_1 b_3); \,1 \right)$. $HTE$ results in $\mathbf{A_3} \sim \mathbf{A_2}$ and $\mathbf{B_3} \sim \mathbf{B_2}$.
\item
$IIP$ provides that $a_2 / (a_1 a_3) \geq 1$ and $b_2 / (b_1 b_3) \geq 1$ can be assumed without loss of generality in $\mathbf{A_3}$ and $\mathbf{B_3}$, respectively.
\item
To summarize, $\mathbf{A_3} \sim \mathbf{A_2} \sim \mathbf{A_1} \sim \mathbf{A} \succeq \mathbf{B} \sim \mathbf{B_1} \sim \mathbf{B_2} \sim \mathbf{B_3}$. It leads to $1 \leq a_2 / (a_1 a_3) \leq b_2 / (b_1 b_3)$ because of the property $PR$.
\end{enumerate}
However, $1 \leq a_2 / (a_1 a_3) \leq b_2 / (b_1 b_3)$ means that $\mathbf{A_1} \succeq^K \mathbf{B_1}$ due to Definition~\ref{Def6}. Analogously, since $\mathbf{A_1} \preceq \mathbf{A'}$ for all triads $\mathbf{A'} \subset \mathbf{A}$ and $\mathbf{B_1} \preceq \mathbf{B'}$ for all triads $\mathbf{B'} \subset \mathbf{B}$, we get $\mathbf{A_1} \preceq^K \mathbf{A'}$ for all triads $\mathbf{A'} \subset \mathbf{A}$ and $\mathbf{B_1} \preceq^K \mathbf{B'}$ for all triads $\mathbf{B'} \subset \mathbf{B}$, respectively. It is equivalent to $\mathbf{A} \sim \mathbf{A_1} \succeq^K \mathbf{B_1} \sim \mathbf{B}$, so $\succeq$ is the Koczkodaj inconsistency ranking.
\end{proof}

The six properties can be classified in at least two ways.
The proof of Theorem~\ref{Theo1} reveals that the first four axioms ($PR$, $IIP$, $HTE$ and $SI$) characterize an inconsistency ranking on the set of triads, while $MON$ and $RED$ are responsible for its extension to pairwise comparison matrices with at least four entities. Therefore, by accepting $MON$ and $RED$, other inconsistency rankings based on the maximally inconsistent triad can be characterized.
From another perspective, two axioms, $PR$ and $MON$ contain a 'preference' relation, whereas $IIP$, $HTE$ and $SI$ define some equivalence classes. $RED$ plays a special role by imposing equivalence between matrices of different size.

The six axioms above ($PR$, $IIP$, $HTE$, $SI$, $MON$ and $RED$) are not enough to uniquely determine the Koczkodaj inconsistency index $I^K$ as any monotonic function of it generates the same inconsistency ranking.
It means that some technical axioms are necessary in order to obtain the Koczkodaj inconsistency index.

\begin{definition}
\emph{Consistency detection} ($CT$):
Consider a consistent triad $\mathbf{T}$.
Inconsistency index $I$ satisfies $CT$ if $I(\mathbf{T}) = 0$.
\end{definition}

$CT$ was suggested by \citet{KoczkodajSzybowski2015} and is a specific version of the property \emph{Existence of a unique element representing consistency} in \citet{BrunelliFedrizzi2015}. Nonetheless, adding $CT$ to the previous properties is still not enough to exclude some transformations of $I^K$. Since it does not make much sense to differentiate between inconsistency indices which generate the same inconsistency ranking \citep{Brunelli2016b}, the axiomatic characterization of the Koczkodaj inconsistency index remains a topic of future research.

\section{Independence of the axioms} \label{Sec5}

According to Theorem~\ref{Theo1}, the six properties determine a unique inconsistency ranking. However, it may turn out that there is some redundancy in the result, certain axioms can be left out. Therefore, six examples, all of them different from the Koczkodaj inconsistency ranking, are given such that they satisfy all but one properties.

\begin{example} \label{Examp1}
Consider two pairwise comparison matrices $\mathbf{A} = \left[ a_{ij} \right] \in \mathbb{R}^{n \times n}_+$ and $\mathbf{B} = \left[ b_{ij} \right] \in \mathbb{R}^{m \times m}_+$. Then $\mathbf{A} \succeq^{1} \mathbf{B}$ if
\begin{equation} \label{eq3}
\min_{i<j<k} \left( \max \left\{ \frac{a_{ij} a_{jk}}{a_{ik}} ; \, \frac{a_{ik}}{a_{ij} a_{jk}} \right\} \right) \geq \min_{i<j<k} \left( \max \left\{ \frac{b_{ij} b_{jk}}{b_{ik}} ; \, \frac{b_{ik}}{b_{ij} b_{jk}} \right\} \right).
\end{equation}
\end{example}

\begin{lemma}
The inconsistency ranking $\succeq^{1}$ in Example~\ref{Examp1} meets $IIP$, $HTE$, $SI$, $MON$ and $RED$, but fails $PR$.
\end{lemma}

\begin{proof}
$IIP$, $HTE$ and $SI$ contain an equivalence relation on the set of triads, and the inconsistency ranking $\succeq^{1}$ shows the same behaviour as the Koczkodaj inconsistency ranking from this point of view.
Inconsistency ranking $\succeq^{1}$ is also based on triads, so it meets $RED$. Monotonicity is satisfied because minimum over a subset cannot be larger than minimum over the original set.
However, it does not satisfy $PR$ since it completely reverses the ranking of triads.
\end{proof}

\begin{example} \label{Examp2}
Consider two pairwise comparison matrices $\mathbf{A} = \left[ a_{ij} \right] \in \mathbb{R}^{n \times n}_+$ and $\mathbf{B} = \left[ b_{ij} \right] \in \mathbb{R}^{m \times m}_+$. Then $\mathbf{A} \succeq^{2} \mathbf{B}$ if
\begin{equation} \label{eq4}
\max_{i<j<k} \frac{a_{ij} a_{jk}}{a_{ik}} \leq \max_{i<j<k} \frac{b_{ij} b_{jk}}{b_{ik}}.
\end{equation}
\end{example}

\begin{lemma}
The inconsistency ranking $\succeq^{2}$ in Example~\ref{Examp2} meets $PR$, $HTE$, $SI$, $MON$ and $RED$, but fails $IIP$.
\end{lemma}

\begin{proof}
$PR$, $HTE$ and $SI$ consider only the upper triangle of a triad, thus they are satisfied.
Inconsistency ranking $\succeq^{2}$ is based on the maximally inconsistent triad, and the set of triads of a pairwise comparison submatrix is a subset of the set of triads of the original matrix, therefore it satisfies $MON$ and $RED$.

Take the following triads:
\[
\mathbf{A} =
\left[
\begin{array}{ccc}
    1     & 1     & 1     \\
    1     & 1     & 1     \\
    1     & 1     & 1     \\
\end{array}
\right]
\qquad
\text{and}
\qquad
\mathbf{B} =
\left[
\begin{array}{ccc}
    1     & 1     & 2     \\
    1     & 1     & 1     \\
     1/2  & 1     & 1     \\
\end{array}
\right].
\]
Now $1 > 1/2$ implies $\mathbf{A} \prec^{2} \mathbf{B}$, but $1 < 2$, so $\mathbf{A}^\top \succ^{2} \mathbf{B}^\top$, showing the violation of invariance under inversion of preferences.
\end{proof}

\begin{example} \label{Examp3}
Consider two pairwise comparison matrices $\mathbf{A} = \left[ a_{ij} \right] \in \mathbb{R}^{n \times n}_+$ and $\mathbf{B} = \left[ b_{ij} \right] \in \mathbb{R}^{m \times m}_+$. Then $\mathbf{A} \succeq^{3} \mathbf{B}$ if
\begin{equation} \label{eq5}
\max_{i<j<k} \left( \max \left\{ \frac{a_{jk}^2}{a_{ik}} ; \, \frac{a_{ik}}{a_{jk}^2} \right\} \right) \leq \max_{i<j<k} \left( \max \left\{ \frac{b_{jk}^2}{b_{ik}} ; \, \frac{b_{ik}}{b_{jk}^2} \right\} \right).
\end{equation}
\end{example}

\begin{lemma}
The inconsistency ranking $\succeq^{3}$ in Example~\ref{Examp3} meets $PR$, $IIP$, $SI$, $MON$ and $RED$, but fails $HTE$.
\end{lemma}

\begin{proof}
$PR$ is satisfied because of the analogous role of $a_{ik}$ in \eqref{eq2} and \eqref{eq5}.
$IIP$ and $SI$ is met due to the substitution of $a_{ij}$ with $a_{jk}$.
Inconsistency ranking $\succeq^{3}$ is based on the maximally inconsistent triad, and the set of triads of a pairwise comparison submatrix is a subset of the set of triads of the original matrix, therefore it satisfies $MON$ and $RED$.

Take the following triads:
\[
\mathbf{A} =
\left[
\begin{array}{ccc}
    1     & 1     & 3     \\
    1     & 1     & 2     \\
     1/3  &  1/2  & 1     \\
\end{array}
\right]
\qquad
\text{and}
\qquad
\mathbf{B} =
\left[
\begin{array}{ccc}
    1     & 1     & 5     \\
    1     & 1     & 4     \\
     1/5  &  1/4  & 1     \\
\end{array}
\right].
\]
Now $2^2 / 3 < 4^2 / 5$ implies $\mathbf{A} \succ^{3} \mathbf{B}$. A possible transformation of triads $\mathbf{A}$ and $\mathbf{B}$ according to $HTE$ leads to $\mathbf{A'} = (1; \,3/2; \,1)$ and $\mathbf{B'} = (1; \,5/4; \,1)$, so $\mathbf{A'} \prec^{3} \mathbf{B'}$ since $3/2 > 5/4$.
It reveals the violation of homogeneous treatment of entities by the inconsistency ranking $\succeq^{3}$.
\end{proof}

\begin{example} \label{Examp4}
Consider two pairwise comparison matrices $\mathbf{A} = \left[ a_{ij} \right] \in \mathbb{R}^{n \times n}_+$ and $\mathbf{B} = \left[ b_{ij} \right] \in \mathbb{R}^{m \times m}_+$. Then $\mathbf{A} \succeq^{4} \mathbf{B}$ if
\begin{equation} \label{eq6}
\max_{i<j<k} \left( \max \left\{ \frac{a_{jk}}{a_{ik}} ; \, \frac{a_{ik}}{a_{jk}} \right\} \right) \leq \max_{i<j<k} \left( \max \left\{ \frac{b_{jk}}{b_{ik}} ; \, \frac{b_{ik}}{b_{jk}} \right\} \right).
\end{equation}
\end{example}

\begin{lemma}
The inconsistency ranking $\succeq^{4}$ in Example~\ref{Examp4} meets $PR$, $IIP$, $HTE$, $MON$ and $RED$, but fails $SI$.
\end{lemma}

\begin{proof}
$PR$ is satisfied because of the analogous role of $a_{ik}$ in \eqref{eq2} and \eqref{eq6}.
$IIP$ and $HTE$ hold as only $a_{ij}$ is eliminated from the original definition.
Inconsistency ranking $\succeq^{4}$ is based on the maximally inconsistent triad, and the set of triads of a pairwise comparison submatrix is a subset of the set of triads of the original matrix, therefore it satisfies $MON$ and $RED$.

Take the following triads:
\[
\mathbf{A} =
\left[
\begin{array}{ccc}
    1     & 1     & 1     \\
    1     & 1     & 1     \\
    1     & 1     & 1     \\
\end{array}
\right]
\qquad
\text{and}
\qquad
\mathbf{A'} =
\left[
\begin{array}{ccc}
    1     & 2     & 4     \\
     1/2  & 1     & 2     \\
     1/4  &  1/2  & 1     \\
\end{array}
\right].
\]
Now $1/1 > 2/4$ implies $\mathbf{A} \prec^{4} \mathbf{A'}$, but $\mathbf{A'}$ can be obtained by a transformation of triad $\mathbf{A}$ according to $SI$, showing the violation of scale invariance.
\end{proof}

\begin{example} \label{Examp5}
Consider two pairwise comparison matrices $\mathbf{A} = \left[ a_{ij} \right] \in \mathbb{R}^{n \times n}_+$ and $\mathbf{B} = \left[ b_{ij} \right] \in \mathbb{R}^{m \times m}_+$. Then $\mathbf{A} \succeq^{5} \mathbf{B}$ if
\begin{equation} \label{eq7}
\min_{i<j<k} \left( \max \left\{ \frac{a_{ij} a_{jk}}{a_{ik}} ; \, \frac{a_{ik}}{a_{ij} a_{jk}} \right\} \right) \leq \min_{i<j<k} \left( \max \left\{ \frac{b_{ij} b_{jk}}{b_{ik}} ; \, \frac{b_{ik}}{b_{ij} b_{jk}} \right\} \right).
\end{equation}
\end{example}

\begin{lemma}
The inconsistency ranking $\succeq^{5}$ in Example~\ref{Examp5} meets $PR$, $IIP$, $HTE$, $SI$ and $RED$, but fails $MON$.
\end{lemma}

\begin{proof}
Inconsistency ranking $\succeq^{5}$ is equivalent to the Koczkodaj inconsistency ranking on the set of triads.
$PR$, $IIP$, $HTE$ and $SI$ consider only a triad, hence these properties are satisfied.
There exists a minimally inconsistent triad in any pairwise comparison matrix, so $RED$ is met.
The set of triads of a pairwise comparison submatrix is a subset of the set of triads of the original matrix, therefore inconsistency ranking $\succeq^{5}$ violates $MON$.
\end{proof}

\begin{example} \label{Examp6}
Consider two pairwise comparison matrices $\mathbf{A} = \left[ a_{ij} \right] \in \mathbb{R}^{n \times n}_+$ and $\mathbf{B} = \left[ b_{ij} \right] \in \mathbb{R}^{m \times m}_+$. Then $\mathbf{A} \succeq^{6} \mathbf{B}$ if
\begin{equation} \label{eq8}
\max_{i<j<k} \left( \max \left\{ \frac{a_{ij} a_{jk}}{a_{ik}} ; \, \frac{a_{ik}}{a_{ij} a_{jk}} \right\} \right)^n \leq \max_{i<j<k} \left( \max \left\{ \frac{b_{ij} b_{jk}}{b_{ik}} ; \, \frac{b_{ik}}{b_{ij} b_{jk}} \right\} \right)^m.
\end{equation}
\end{example}

\begin{lemma}
The inconsistency ranking $\succeq^{6}$ in Example~\ref{Examp6} meets $PR$, $IIP$, $HTE$, $SI$ and $MON$, but fails $RED$.
\end{lemma}

\begin{proof}
Inconsistency ranking $\succeq^{6}$ is equivalent to the Koczkodaj inconsistency ranking when $n=m$, for example, on the set of triads.
Thus $PR$, $IIP$, $HTE$ and $SI$ are all satisfied.
Since $\max \{ a_{ij} a_{jk} / a_{ik} ; \, a_{ik} / (a_{ij} a_{jk}) \} \geq 1$, inconsistency ranking $\succeq^{6}$ meets $MON$.
However, it violates $RED$ because of the appearance of $n$ and $m$ in \eqref{eq8}.
\end{proof}

Table~\ref{Table1} summarizes the above discussion on the independence of the axioms.

\begin{table}[htbp]
\centering
\caption{Axiomatic properties of inconsistency rankings}
\label{Table1}
\noindent\makebox[\textwidth]{
	\begin{tabularx}{1\textwidth}{ll CCCCCC} \toprule
   \parbox[m]{2.5cm}{Inconsistency \\ ranking} & Definition & $PR$ & $IIP$ & $HTA$ & $SI$ & $MON$ & $RED$ \\ \midrule
    Koczkodaj & Definition~\ref{Def7} & \textcolor{PineGreen}{\ding{52}} & \textcolor{PineGreen}{\ding{52}} & \textcolor{PineGreen}{\ding{52}} & \textcolor{PineGreen}{\ding{52}} & \textcolor{PineGreen}{\ding{52}} & \textcolor{PineGreen}{\ding{52}} \\ \midrule
    $\succeq^{1}$ & Example~\ref{Examp1} & \textcolor{BrickRed}{\ding{55}} & \textcolor{PineGreen}{\ding{52}} & \textcolor{PineGreen}{\ding{52}} & \textcolor{PineGreen}{\ding{52}} & \textcolor{PineGreen}{\ding{52}} & \textcolor{PineGreen}{\ding{52}} \\
    $\succeq^{2}$ & Example~\ref{Examp2} & \textcolor{PineGreen}{\ding{52}} & \textcolor{BrickRed}{\ding{55}} & \textcolor{PineGreen}{\ding{52}} & \textcolor{PineGreen}{\ding{52}} & \textcolor{PineGreen}{\ding{52}} & \textcolor{PineGreen}{\ding{52}} \\
    $\succeq^{3}$ & Example~\ref{Examp3} & \textcolor{PineGreen}{\ding{52}} & \textcolor{PineGreen}{\ding{52}} & \textcolor{BrickRed}{\ding{55}} & \textcolor{PineGreen}{\ding{52}} & \textcolor{PineGreen}{\ding{52}} & \textcolor{PineGreen}{\ding{52}} \\
    $\succeq^{4}$ & Example~\ref{Examp4} & \textcolor{PineGreen}{\ding{52}} & \textcolor{PineGreen}{\ding{52}} & \textcolor{PineGreen}{\ding{52}} & \textcolor{BrickRed}{\ding{55}} & \textcolor{PineGreen}{\ding{52}} & \textcolor{PineGreen}{\ding{52}} \\
    $\succeq^{5}$ & Example~\ref{Examp5} & \textcolor{PineGreen}{\ding{52}} & \textcolor{PineGreen}{\ding{52}} & \textcolor{PineGreen}{\ding{52}} & \textcolor{PineGreen}{\ding{52}} & \textcolor{BrickRed}{\ding{55}} & \textcolor{PineGreen}{\ding{52}} \\
    $\succeq^{6}$ & Example~\ref{Examp6} & \textcolor{PineGreen}{\ding{52}} & \textcolor{PineGreen}{\ding{52}} & \textcolor{PineGreen}{\ding{52}} & \textcolor{PineGreen}{\ding{52}} & \textcolor{PineGreen}{\ding{52}} & \textcolor{BrickRed}{\ding{55}} \\ \bottomrule
    \end{tabularx} }
\end{table}

Note that inconsistency rankings $\succeq^{1}$ and $\succeq^{5}$ differ only in the direction of the inequality.
Analogously to Theorem~\ref{Theo1}, $\succeq^{1}$ is characterized by $IIP$, $HTE$, $SI$, $MON$, $RED$ (the five axioms it satisfies) and negative responsiveness, the inverse of positive responsiveness defined as $\mathbf{S} \succeq \mathbf{T}$ if and only if $t_2 \geq s_2$.
In a similar way, $\succeq^{5}$ is the unique inconsistency ranking satisfying $PR$, $IIP$, $HTE$, $SI$, $RED$ and reversed monotonicity, that is, $\mathbf{A} \succeq \mathbf{B}$ for any pairwise comparison matrix $\mathbf{A}$ and its triad $\mathbf{B}$.

Inconsistency rankings $\succeq^{1}$--$\succeq^{5}$ have not much practical use. On the other hand, inconsistency ranking $\succeq^{6}$ may be meaningful if inconsistency is assumed to (deterministically) increase with size due to the possible presence of other inconsistent triads.

\section{Discussion} \label{Sec6}

The paper has provided an axiomatic characterization of the inconsistency ranking induced by the Koczkodaj inconsistency index. The theorem requires six properties, some of them directly borrowed from previous axiomatic discussions ($IIP$ from \citet{Brunelli2017} and $MON$ from \citet{KoczkodajSzybowski2015}), some of them similar to the axioms used in the literature ($PR$ is a relaxed version of monotonicity on single comparisons by \citet{BrunelliFedrizzi2015} and $RED$ is implicitly contained in monotonicity by \citet{KoczkodajSzybowski2015}), and some of them introduced here ($HTE$ and $SI$). Although we have given a few arguments for homogeneous treatment of entities and scale invariance, their justification can be debated.
However, it does not influence the essence of our characterization, which only states that the Koczkodaj inconsistency ranking should be accepted when \emph{all} properties are adopted.

It remains to be seen how this characterization can be modified. The proof of Theorem~\ref{Theo1} shows that $IIP$ can be eliminated if $PR$ is defined without the condition $s_2, t_2 \geq 1$. Although it seems to be a simplification (we will have only five axioms instead of six), the current form better reflects the role of different properties, highlighted by the discussion of logical independence in Section~\ref{Sec5}.

Another direction can be the use of other known properties in the characterization. A natural candidate is the following one, suggested by \citet{BrunelliFedrizzi2015} and applied by \citet{KoczkodajSzybowski2015}, too.

\begin{definition}
\emph{Order invariance} ($OI$):
Consider a pairwise comparison matrix $\mathbf{A}$ and a permutation matrix $\mathbf{P}$ of the same size. Inconsistency ranking $\succeq$ satisfies $OI$ if $\mathbf{PAP^\top} \sim \mathbf{A}$.
\end{definition}

$OI$ says that changing the order of the entities does not affect the inconsistency of preferences.
In our characterization, homogeneous treatment of entities and scale invariance have a somewhat analogous role, but they cannot be substituted immediately by $OI$.

To summarize, the current study clearly does not mean the end of the discussion on inconsistency indices. 
The main aim of this topic is perhaps the identification of the best universal measure, or at least, the perfect in a certain field. For this purpose, a thorough understanding of inconsistency indicators is indispensable. We hope the current study may give new insights by presenting the first axiomatic characterization in this field for the ranking derived from the Koczkodaj inconsistency index.

\section*{Acknowledgements}
\addcontentsline{toc}{section}{Acknowledgements}
\noindent
I would like to thank to Matteo Brunelli, Michele Fedrizzi, Waldemar W. Koczkodaj and Jacek Szybowski for inspiration.
I am also grateful to Matteo Brunelli, S\'andor Boz\'oki and Mikl\'os Pint\'er for reading the manuscript and for useful advices. \\
Three anonymous reviewers provided valuable comments and suggestions on earlier drafts. \\
The research was supported by OTKA grant K 111797 and by the MTA Premium Post Doctorate Research Program. \\
This research was partially supported by Pallas Athene Domus Scientiae Foundation. The views expressed are those of the author's and do not necessarily reflect the official opinion of Pallas Athene Domus Scientiae Foundation.


\end{document}